\documentclass[letterpaper, 10 pt, conference]{ieeeconf}  

\IEEEoverridecommandlockouts                              

\usepackage{amsmath,amsfonts,amssymb,color,amsthm}
\usepackage{nccmath}
\usepackage{epsfig}
\usepackage{psfrag}
\usepackage{algorithm}
\usepackage{algpseudocode}
\usepackage{array}
\usepackage{epstopdf}
\usepackage{cite}
\usepackage{medmath}
\usepackage{footmisc}
\usepackage{mathtools}
\usepackage{bbm}
\usepackage{bm}
\usepackage{comment}
\usepackage{dsfont}

\usepackage{mathtools}
\usepackage{epstopdf}
\usepackage{tikz}
\usetikzlibrary{automata, positioning, arrows.meta}
\usepackage{relsize}
\usetikzlibrary{shapes,arrows}
\usepackage{cite}
\usepackage{epstopdf}

\definecolor{winered}{rgb}{0.5,0,0}
\usepackage{scalerel}
\usepackage{xcolor}
\usepackage[colorlinks]{hyperref}

\newcommand{\mc}{\mathcal}

\AtBeginDocument{%
  \hypersetup{
    citecolor=blue,
    linkcolor=winered}}
\DeclareMathOperator*{\argmax}{\arg\!\max}

\newcommand{\E}{\mathbb{E}}

\newcommand{\cA}{\mathcal{A}}

\newcommand{\cE}{\mathcal{E}}

\newcommand{\cG}{\mathcal{G}}

\newcommand{\cM}{\mathcal{M}}
\newcommand{\cN}{\mathcal{N}}
\newcommand{\cO}{\mathcal{O}}
\newcommand{\cP}{\mathcal{P}}

\newcommand{\cR}{\mathcal{R}}
\newcommand{\cS}{\mathcal{S}}
\newcommand{\cT}{\mathcal{T}}

\newcommand{\cV}{\mathcal{V}}

\newcommand{\bracket}[1]{\left[ #1\right]}
\newcommand{\norm}[1]{\left\| #1\right\|_{\infty}}

\newcommand{\bigot}[1]{\tilde\cO\left( #1\right)}

\newtheorem{theorem}{Theorem}

\newtheorem{lemma}{Lemma}
\newtheorem{remark}{Remark}

\newtheorem{assumption}{Assumption}

\newtheorem{fact}{Fact}

\DeclarePairedDelimiter\ceil{\lceil}{\rceil}

\def\BibTeX{{\rm B\kern-.05em{\sc i\kern-.025em b}\kern-.08em
    T\kern-.1667em\lower.7ex\hbox{E}\kern-.125emX}}

\newcommand{\stackref}[2]{
\readlist*\mylist{#1}
\stackrel{\mbox{\footnotesize\foreachitem\x\in\mylist[]{\ifnum\xcnt=1\else,\fi\eqref{\x}}}}{#2}
}

\title{\LARGE \bf Variance-Reduced Q-Learning over Static and Time-Varying Networks
}

\author{Sreejeet Maity\textsuperscript{\textdagger,*},
        Feng Zhu\textsuperscript{\textdagger,*},
        Aritra Mitra\textsuperscript{*},
        and Robert W. Heath Jr.
\thanks{\textsuperscript{$\dagger$}Equal Contribution. 
\textsuperscript{*}The authors are with the Department of Electrical and Computer Engineering, North Carolina State University. Email: \texttt{\{smaity2,fzhu5,amitra2\}@ncsu.edu}. 
Robert W. Heath Jr. is with the Department of Electrical and Computer Engineering, University of California San Diego, San Diego, USA. Email: \texttt{rwheathjr@ucsd.edu}.}
}

\begin{document}

\maketitle
\thispagestyle{empty}
\pagestyle{empty}

\begin{abstract}
We investigate a decentralized reinforcement learning problem involving multiple agents that interact with the same Markov Decision Process (MDP). The agents can exchange information over a network to collectively learn the optimal state-action value function. For this setting, we introduce a novel epoch-based  distributed $Q$-learning algorithm called \textcolor{winered}{\texttt{VRDQ}}, where within each epoch, agents locally estimate the Bellman optimality operator and diffuse information using a consensus-based protocol. For both static and time-varying networks, we establish high-probability finite-time convergence rates for \textcolor{winered}{\texttt{VRDQ}} that enjoy linear speedups from collaboration. Crucially, we prove that such speedups in sample-complexity require only $\tilde{O}(1)$ communication, substantially improving upon the communication costs in prior work. 
\end{abstract}

\section{Introduction}
Given the promise of cooperative multi-agent reinforcement learning (MARL) in improving the sample efficiency of online decision-making, we consider a decentralized RL setting involving $N$ agents that can exchange information over a (potentially time-varying) network. Each agent interacts with a \emph{common environment}  modeled as a Markov Decision Process (MDP), with the goal of learning an optimal policy that maximizes a long-term cumulative return. In the single-agent setting, such an optimal policy can be learned using the celebrated $Q$-learning algorithm~\cite{watkins1992q}. Given the collective information available across the network in our setting, we focus on answering two basic questions: {(i)}~(\textbf{Sample Efficiency}) By exchanging information, can each agent learn the optimal policy using fewer samples (relative to the single-agent case)? {(ii)}~(\textbf{Communication Efficiency}) If so, what is the \emph{communication overhead} required to realize such collaborative speedups? Perhaps surprisingly, as we discuss below, these questions remain unresolved even for simple tabular RL problems, thereby motivating our current study. 

\textbf{Related Work.} Asymptotic convergence guarantees for a decentralized variant of the classical $Q$-learning algorithm were first provided in~\cite{soumyadecentralized}; similar results for actor-critic algorithms were later derived in~\cite{zhang2018fully, zavlanos}. However, to characterize explicit statistical gains from collaboration, one requires a finer non-asymptotic analysis absent in these papers. In a series of follow-up papers, finite-time rates were derived for decentralized temporal difference learning in~\cite{doan2019finite}, $Q$-learning in~\cite{heredia2020finite, lim2025finite}, and general stochastic approximation in~\cite{zeng2022decentralized}. However, there are two key limitations to all the aforementioned works. First, the final performance bounds do not explicitly demonstrate any benefit of collaboration between agents. Moreover, each method incurs a communication cost that scales linearly with the time horizon, i.e., the total number of samples. This creates a natural tension: \emph{with existing decentralized RL methods, achieving a high accuracy by gathering more samples comes at the expense of commensurately high communication costs, hindering their practical deployment in resource-constrained environments.}  

In this work, we resolve the above tension by developing a novel distributed Q-learning algorithm that (i) enjoys \emph{near-optimal} statistical benefits of collaboration, and (ii) incurs a communication cost that scales only \emph{poly-logarithmically} in the total number of samples, marking a significant improvement over existing methods that require linear-in-time communication costs. Our results apply to both static and time-varying networks, and, as such, are significantly more general in scope relative to some recent papers that assume a central coordinator~\cite{khodadadian, woo2023blessing, wang2023TMLR}. Importantly, our proposed algorithm has a structure that is fundamentally different from all the papers we have reviewed thus far.  

\noindent \(\bullet\) \textbf{Algorithmic Contribution.} We introduce a new decentralized RL algorithm called Variance-Reduced Diffused $Q$-learning (\texttt{VRDQ}) that combines two crucial ingredients: \emph{local operator estimation} and \emph{diffusion}. In contrast to standard approaches that update the \(Q\)-function (or other relevant parameters) at \emph{every} time-step using noisy update directions that suffer from high variance, our approach relies on making fewer \emph{infrequent} updates using low-variance directions obtained by locally estimating the Bellman optimality operator. The low frequency of updates, in turn, directly translates into the low communication overhead of our algorithm. The second key ingredient of our approach is to show how the local operator estimation phase can be run in parallel with an average-consensus-based diffusion phase meant to exchange information. The decoupled nature of these phases allows us to easily disentangle statistical errors from network-induced errors, leading to a simple overall analysis. 

\noindent \(\bullet\) \textbf{Theoretical Contribution.} Our first main result, namely Theorem~\ref{thm:main_result}, pertains to static networks, and establishes a high-probability finite-sample convergence rate of \(\tilde{\mc{O}} (1/\sqrt{NT})\) for our proposed algorithm \texttt{VRDQ}, where $T$ is the number of samples per agent, and $N$ is the number of agents. This result reveals a clear benefit of collaboration over the single-agent rate of \(\tilde{\mc{O}} (1/\sqrt{T})\) recently established in~\cite{Waiwright, Qu, li2024q}. In Theorem~\ref{thm:timevar}, we prove that \texttt{VRDQ} continues to enjoy the same collaborative gains for a fairly general class of time-varying networks. Crucially, for both static and time-varying networks, we show that such collaborative gains can be achieved with just $\mc{O}(\log^2(NT))$ communication. 

\section{Notation and Problem Formulation}\label{sec:notation}
\textbf{Graph model.}  We start by introducing our network model. Let $\cV = \{1, 2, \ldots, N\}$ be a set of $N$ agents interacting with the same environment (modeled as an MDP). The network is an undirected graph $\cG = (\cV, \cE)$, where $\cV$ is the set of nodes (agents), and $\cE \subseteq \cV \times \cV$ is the set of edges representing communication links between agents. Since the graph is undirected, we have $(i, j) \in \cE$ if and only if $(j, i) \in \cE$. We say that agent $j$ is a \textit{neighbor} of agent $i$ if $(i, j) \in \cE$, and define the \textit{neighbor set} of agent $i$ as the set of all its neighbors (including itself): $\cN_i=\{j\mid(i,j)\in\cE\}$. We associate a \textit{mixing matrix} $W \in \mathbb{R}^{N \times N}$ with the network, where the entry $(W)_{ij}$ denotes the weight that agent $i$ assigns to agent $j$'s information. If $i \neq j$, and $(i,j) \notin \mathcal{E}$, then $(W)_{ij}=0.$ The mixing matrix $W$ plays a central role in distributed algorithms, as it governs how agents combine information from their neighbors. As is standard, we assume that $W$ is symmetric and \textit{doubly stochastic}, i.e., $W = W^\top$, all entries are non-negative, and each row and column sums to one. Later, in Section~\ref{sec:TimeVar}, we will see how our results can be easily generalized to time-varying networks.

\textbf{MDP Model.} We now introduce the basic MDP notation used in this paper. We assume that the agents in $\cV$ interact with the same environment, which can be modeled as an MDP $\cM=\{\cS,\cA,\cR,\cP,\gamma\}$, where $\cS$ and $\cA$ are the finite state and action spaces whose cardinalities are denoted as $S$ and $A$, and $\cR:\cS\times\cA\to \mathbb{R}$ is the reward function, with $\cR(s,a)$ denoting the immediate deterministic reward received by playing action $a$ at state $s$. Throughout this paper, we assume bounded rewards with $|\cR(s,a)|\leq\bar R$, where $\bar R > 0$ is some constant. The object $\cP:\cS\times\cA\times\cS\to[0, 1]$ is the transition kernel, with $\cP(s'\mid s, a)$ denoting the probability of transitioning to the next state $s'$ by playing action $a$ at state $s$. Finally, $\gamma\in(0,1)$ is the discount factor.

The behavior of an agent is captured by a stochastic \textit{policy} $\pi:\cS\to\Delta(\cA)$, which outputs a probability distribution over the action space $\cA$ at a given state $s\in\cS$. It is then natural to develop a metric to measure the \textit{goodness} of a policy $\pi$ when the agent interacts with the MDP $\cM$ following this policy. This leads to the concept of the \textit{value function} $V^\pi\in\mathbb{R}^S$, defined as the expected infinite-horizon cumulative discounted reward starting from state $s\in\cS$:
\begin{equation}
    V^\pi(s)=\E\bracket{\sum_{t=0}^\infty\gamma^t\cR(s_t, a_t)\ \bigg |\ s_0=s, \pi},\label{eqn:value_function}
\end{equation}
where the expectation is taken over the randomness w.r.t. state transitions and the stochastic policy $\pi$; here, $s_t$ and $a_t$ represent the state and action, respectively, at time $t$.  Similarly, we introduce the concept of the \textit{state-action value function}, or the \textit{\(Q\)-function} $Q^\pi\in\mathbb{R}^{S\times A}$, which evaluates the policy starting from state $s$ under initial action $a$:
\begin{equation}
    Q^\pi(s,a)=\cR(s,a) + \E\bracket{\sum_{t=1}^\infty\gamma^t\cR(s_t, a_t)\ \bigg |\ \pi}.\label{eqn:Q_function}
\end{equation}

\textbf{\(Q\)-Learning.} The general objective of an agent is to find the optimal policy $\pi^*$ that maximizes the value function $V^\pi$ for all states $s\in\cS$. Unlike in dynamic programming problems, in our RL setup, the underlying MDP model comprising the transition kernels and the reward functions is \textit{unknown} to the agent. In this setting, the celebrated \(Q\)-learning algorithm~\cite{watkins1992q} learns $\pi^*$ by first iteratively estimating the optimal \(Q\)-function $Q^*\in\mathbb{R}^{S \times A}$, and then extracting an associated optimal policy $\pi^*$ via greedy action selection: $\pi^*(s)=\argmax_{a\in\cA} Q^*(s, a)$ at each state $s\in\cS$. The core idea behind \(Q\)-learning is to exploit the fact that $Q^*$ is the unique fixed-point of the \textit{Bellman optimality operator} $\cT^*:\mathbb{R}^{S\times A}\to \mathbb{R}^{S\times A}$ defined below~\cite{suttonRL}:
\begin{equation}\label{eqn:bellman_optimality}
    \cT^*f(s,a):=\cR(s,a) + \gamma\sum_{s'\in\cS}\cP(s'\mid s,a)\max_{a'\in\cA}f(s',a'),
\end{equation}
$\forall f\in\mathbb{R}^{S\times A}$. To run $Q$-learning, one maintains a sequence of noisy empirical estimates of $\cT^*$ using data generated by a suitable sampling model. In this work, we will consider the popular \textit{generative synchronous sampling model}~\cite{Waiwright, li2024q, kearns, sidford} where an agent makes observations of the following form. At each time step $t=0,1,\ldots$, for \textbf{every} state-action pair $(s,a)\in \mc{S}\times \mc{A}$, the agent independently samples a next state $s_t(s,a) \sim\cP(\cdot\mid s,a)$, and observes an immediate deterministic reward $\cR(s,a)$. The agent then constructs a \textit{noisy empirical} estimate $\cT_t:\mathbb{R}^{ S\times A}\to \mathbb{R}^{ S\times A}$ of the Bellman optimality operator $\cT^*$, defined as follows:
\begin{equation}
    \cT_t f(s,a) =\cR(s,a) + \gamma\max_{a'\in\cA}f(s_t(s,a),a'),\label{eqn:noisy_T}
\end{equation}
$\forall f\in\mathbb{R}^{S\times A}$. Using $\cT_t$, the \(Q\)-learning algorithm updates the Q-estimate as follows $\forall (s,a)\in\cS\times\cA$:
\begin{equation}
    Q_{t+1}(s,a) = (1-\alpha_t)Q_t(s,a) + \alpha_t \cT_t Q_t(s,a),\label{eqn:update_Q}
\end{equation}
where $Q_t\in\mathbb{R}^{S\times A}$ is the estimated \(Q\)-function at time-step $t$, and $\{\alpha_t\}$ is a suitable step-size sequence. Classical asymptotic results show that the sequence of iterates $\{Q_t\}$ generated by~\eqref{eqn:update_Q} converges to $Q^*$ almost surely~\cite{tsitsiklis94}. More recent work~\cite{Waiwright, Qu, li2024q} has established non-asymptotic guarantees, revealing that with $T$ samples, the estimation error $\Vert Q_T - Q^* \Vert_{\infty}$ is $\bigot{1/\sqrt{T}}$, with high probability. 

\textbf{Networked \(Q\)-Learning Problem.} We can now state our problem of interest. Suppose that each agent in $\cV$ is allowed to acquire, in parallel,  $T$ \emph{statistically independent} samples per state-action pair, through $T$ queries to a generative model. Acting independently, each agent can trivially achieve a convergence rate of $\bigot{1/\sqrt{T}}$, following standard results from single-agent \(Q\)-learning. Given that there are $N$ agents in total, we ask: \emph{Can agents collaborate to achieve an accelerated rate of $\bigot{1/\sqrt{NT}}$?} This question is non-trivial due to the decentralized nature of the setting: agents communicate over a graph that is \textbf{not assumed to be fully connected}, i.e., $\mathcal{N}_i \neq \cV$ in general. Therefore, to achieve a linear speedup w.r.t. $N$, agents must \emph{diffuse} their information throughout the network to collectively approximate $Q^*$ at an improved $\bigot{1/\sqrt{NT}}$ rate. We further require such a bound to hold with probability at least $1-\delta$, where $\delta \in (0,1)$ is a prescribed failure probability. 

In this context, we ask two basic questions: \textbf{(i)} \emph{What} information should each agent diffuse with others? \textbf{(ii)} How \emph{frequently} should such information be diffused? We provide concrete answers to these questions in the following section, where we present our algorithmic framework in detail. Before doing so, the following remark is in order. 

\begin{remark} While one can certainly consider more involved RL settings accounting for function approximation and/or asynchronous, Markovian sampling, we restrict our attention to a tabular setting with synchronous sampling for the following reasons. First, the generative synchronous sampling model we consider here has been extensively used in the theoretical analysis of RL algorithms~\cite{Waiwright, li2024q, kearns, sidford}; in a similar vein, tabular Q-learning has been used for building key insights in both single~\cite{Waiwright, Qu, li2024q} and multi-agent RL~\cite{woo2023blessing}. Second, the tabular setting we study here enables us to better convey the new algorithmic aspects of our approach. Third, and most importantly, a fundamental understanding of the amount of communication needed to achieve statistical collaborative gains under network constraints is not well understood for the setting in our paper. 
\end{remark}

\section{Variance-Reduced Diffused Q-Learning}\label{sec:algo}

\begin{algorithm}[H]
\caption{\small{Variance-Reduced Diffused \(Q\)-Learning} (\texttt{VRDQ})}
\label{algo:Algo 2}
\begin{algorithmic}[1]
\Require Total samples per agent $T$, epoch length \(H\), diffusion period \(L\), failure probability $\delta$. 
\State Initialize $Q_{i,0}(s,a) \leftarrow 0, \forall (s,a) \in \mc{S} \times \mc{A}$ and $i\in[N]$.
\For{epoch $k = 0$ to $K-1$}
\For {$i \in [N]$}
    \State \textcolor{winered}{\texttt{Estimate}} $\cT_{i,k}$ as in~\eqref{eqn:agent_wise_bellman_update}. 
    \State \textcolor{winered}{\texttt{Diffuse}} \( d_{i,k}^{(0)}(s,a) \) by running consensus for $L$ steps as per~\eqref{eqn:diffusion_step}.
    \State \texttt{Update} local estimate \(Q_{i,k}\) as per~\eqref{eqn:agent_update}.
\EndFor
\EndFor
\end{algorithmic}
\end{algorithm}

In this section, we introduce our proposed algorithm~\texttt{VRDQ} (outlined as Algorithm~\ref{algo:Algo 2}), which, given a prescribed failure probability $\delta \in (0,1)$, aims to generate local Q-function estimates at every agent that enjoy (due to collaboration) an error rate of $\tilde{O}(1/\sqrt{NT})$ with probability at least $1-\delta.$ Here, $T$ is the number of samples at each agent; since these samples are acquired in parallel, one can also think of $T$ as the total run-time duration of the algorithm. The key guiding observations behind our algorithm are as follows. We note that agents need to exchange information essentially only when their Q-tables are locally updated. So, it is natural to then ask how infrequently the Q-tables can be updated, while preserving the optimal $\tilde{O}(1/\sqrt{NT})$ rate. It turns out that just $\tilde{O}(1)$ updates suffice, provided each update is made using a ``low-variance" estimate of the Bellman optimality operator. Accordingly, our algorithm runs in epochs, where within an epoch, an agent \emph{simultaneously} performs (i) \emph{local estimation} of the Bellman optimality operator, and (ii) \emph{diffusion} of an update direction generated in the previous epoch. Importantly, an update to the Q-table is made \emph{only once} at the end of each epoch. As we shall see, this approach incurs only a logarithmic communication overhead and admits a straightforward analysis. We now describe the key ideas. 

\begin{figure}[t]
\begin{center}
\includegraphics[scale=0.425]{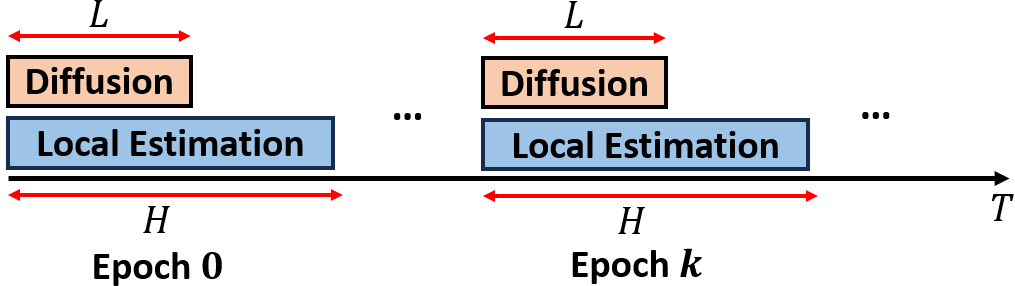}
\vspace{-4mm}
\end{center}
\caption{Illustration of \texttt{VRDQ} which runs in epochs of length $H$. Throughout the duration of each epoch, every agent locally estimates the Bellman optimality operator using samples acquired from the generative sampling model. In parallel, agents run an average consensus protocol for the first $L$ steps of the epoch to diffuse information.}
\label{fig:Algo}
\vspace{-5mm}
\end{figure}

\textbf{\textcolor{winered}{Module 1}} \textbf{(Local Operator Estimation).} In conventional single-agent and distributed RL algorithms, agents typically construct a noisy estimate $\mathcal{T}_t$ (as defined in~\eqref{eqn:noisy_T}) of the Bellman operator  immediately upon observing a new sample, and subsequently update their Q-tables using the update rule in~\eqref{eqn:update_Q}. Inspecting the noisy operator $\cT_t$ in~\eqref{eqn:noisy_T}, we notice that it is constructed using only a single sampled next state, as opposed to the true Bellman optimality operator $\cT^*$ in~\eqref{eqn:bellman_optimality}, which takes an expectation over all possible next states.  
Thus, in standard approaches, agents update their Q-values \textbf{all the time} using \emph{high-variance} noisy operator estimates. Given this observation, we ask: \emph{Does it suffice for each agent to update its Q-table intermittently with a refined (low-variance) operator to achieve the same performance?}

Our main innovation is to show that this is indeed possible via a variance reduction technique. We divide the $T$ samples of each agent into $K$ epochs of $H$ samples each, such that $T=KH$ (assuming $T$ is divisible by $K$ for simplicity); here, $K$ is a design parameter that will be specified later. At epoch $k\in\{0,\cdots, K-1\}$, agent $i\in[N]$ generates an estimate $\cT_{i,k}$ of the Bellman optimality operator $\cT^*$, by estimating the transition kernel $\cP$ and the reward function $\cR$. To do so, the agent sequentially queries the generative sampling model, which provides $H$ independent observations for every state-action pair within each epoch. Since we assume the rewards to be non-noisy, the reward function can be directly estimated as $\cR_{i,k}(s,a)=\cR(s,a)$ for all $(s,a)$ pairs, where $\cR_{i,k}$ is the estimate of the reward function of agent $i$ at epoch $k$. The kernel $\cP$ is estimated as follows: 
\begin{equation}\label{eqn:empirical_prob_dist}
    \cP_{i,k}(s'\mid s,a)=(1/H)\sum_{u=0}^{H-1}\mathbf{1}_{i,k,u}^{s,a}(s'),
\end{equation}
where $\cP_{i,k}$ is the transition kernel estimate of agent $i$ at epoch $k$, and $\mathbf{1}_{i,k,u}^{s,a}(s')$ is an indicator function that equals 1 if, at the $u$-th query to state-action pair $(s,a)$ in epoch $k$, agent $i$ observes $s'$ as the next state. Using these estimates, agent $i$ constructs a refined Bellman optimality operator estimate $\cT_{i,k}:\mathbb{R}^{S\times A}\to\mathbb{R}^{S\times A}$, defined as
\begin{equation}\label{eqn:agent_wise_bellman_update}
    \medmath{\cT_{i,k} f(s,a) := \cR_{i,k}(s,a) + \gamma \sum_{s' \in \cS} \cP_{i,k}(s' \mid s,a) \max_{a' \in \cA} f(s',a')},
\end{equation}
$\forall f \in \mathbb{R}^{S \times A}.$ Intuitively, a larger epoch length $H$ indicates a better estimate for $\cT^*$. The choice of this hyperparameter will be detailed in the next section.

\textbf{\textcolor{winered}{Module 2}} \textbf{(Diffusion).} Having generated $\cT_{i,k}$, it is then necessary to diffuse this information across the network to leverage collaborative gains. Our key observation in this regard is to note that \emph{the time it takes to obtain an accurate estimate of the Bellman operator dominates the time taken to spread information}. As such, our \textbf{core idea} is to have the estimation and diffusion phases evolve \emph{in parallel}, in a decoupled manner; see Fig.~\ref{fig:Algo} for an illustration. Formally, in each epoch $k =0 , 1, \ldots,$ every agent $i$ aims to diffuse the object $\mathcal{T}_{i,k-1}Q_{i,k}$, where $Q_{i,k} \in \mathbb{R}^{S \times A}$ is the Q-estimate of agent $i$ at the beginning of epoch $k$, with $Q_{i,0}=0.$ We use the convention that $\cT_{i,-1}f =0, \forall f \in \mathbb{R}^{S \times A}$. To diffuse information, agents run a consensus protocol for only the first $L$ steps of the $H$-length epoch, where $L$ is a parameter to be specified later. The diffusion steps evolve as
\begin{equation}\label{eqn:diffusion_step}
    d_{i,k}^{(\ell+1)}(s,a) = \sum_{j \in \mathcal{N}_i} (W)_{ij} d_{j,k}^{(\ell)}(s,a), \quad\forall(s,a)\in\cS\times\cA,
\end{equation}
for $\ell = 0, \ldots, L-1$, where the weights $\{(W)_{ij}\}$ are entries of the mixing matrix $W$, and $d_{i,k}^{(\ell)} \in \mathbb{R}^{S \times A}$ denotes the diffusion model of agent $i$ at diffusion step $\ell$ within epoch $k$. This object is initialized as
\begin{equation}\label{eqn:initial_diffusion_step}
    d_{i,k}^{(0)}= \mathcal{T}_{i,k-1}Q_{i,k}.
\end{equation} 
By convention, note that \(d_{i,0}^{(0)}= 0\). Define \(d_{k}^{(\ell)}(s,a) := \big(d_{1,k}^{(\ell)}(s,a), \ldots, d_{N,k}^{(\ell)}(s,a)\big)^\top\), for all \(\ell \in \{0,\cdots, L-1\}\). Based on~\eqref{eqn:diffusion_step}, after \(L\) rounds of diffusion in the \(k\)-th epoch, the following holds for each \((s,a) \in \mc{S} \times \mc{A}\):
\begin{equation}\label{eqn:compact_diffusion_step}
    d_{k}^{(L)}(s,a) = W^L d_{k}^{(0)}(s,a).
\end{equation}
Finally, each agent \(i \in [N]\) updates its $Q$-estimate as follows:
\begin{equation}\label{eqn:agent_update}
    Q_{i,k+1} = (1 - \alpha) Q_{i,k} + \alpha d_{i,k}^{(L)},
\end{equation}
where $\alpha \in (0,1)$ is the step-size.
Intuitively, the diffused model $d_{i,k}^{(L)}$ aggregates information from all agents in the network, thereby reducing the variance of each agent's Q estimate and accelerating convergence through collaborative learning. This intuition will be made precise in the next section, where we will formally analyze \texttt{VRDQ}. 
\section{Main Result}
To state our main result, we require the following key property of doubly-stochastic matrices.
\begin{fact}
\label{fact:geom_mixing_uniform}
Let $W\in\mathbb{R}^{N\times N}$ be doubly stochastic. Assume $W$ is \emph{primitive}, i.e., $W$ is irreducible and aperiodic
\textup{(}it suffices that the induced graph is connected and $(W)_{ii}>0$ for all $i\in[N]$\textup{)}. Then there exist constants
$\texttt{C}_1>0$ and $\rho\in(0,1)$ such that, for all $i\in[N]$ and all $t\ge 0$,
\begin{equation}\label{eqn:fact_1}
\Big\|[W^t]_i-\tfrac{1}{N}\mathbf{1}^\top\Big\|
\;\le\; \texttt{C}_1\,\rho^t,
\end{equation}
where $[W^t]_i$ denotes the $i$-th row of $W^t$ and $\mathbf{1}\in\mathbb{R}^N$ is the all-ones vector.
\end{fact}
Next, define the agent-wise epoch-\(k\) error as \(e_{i,k} := \lVert Q_{i,k} - Q^* \rVert_{\infty}\). Our main result for \texttt{VRDQ} is then as follows.
\begin{theorem}\label{thm:main_result} Given any failure probability $\delta \in (0,1)$, suppose the step-size $\alpha$, number of epochs $K$, and the diffusion period \(L\) in Algorithm~\ref{algo:Algo 2} be chosen as follows:
\begin{equation}\medmath{
\begin{aligned}
&\alpha = \frac{\log(NT)}{(1 - \gamma)K}; \hspace{2mm} K = \ceil{ c_1\log(NT)/(1-\gamma)};\\
&L =  \ceil{{\log\!\Big(c_2\,N^{3/2}\sqrt{T}\sqrt{1-\gamma}\Big)}/{\log(1/\rho)}},
\end{aligned}}
\label{eqn:parameter_choice}
\end{equation}
where $c_1,c_2$ are universal constants, and $\rho\in(0,1)$ is as defined in~\eqref{eqn:fact_1}. Then, after $K$ epochs, the following bound holds with probability at least \( 1 - \delta \):
\begin{equation}\label{eqn:thm_main}
\begin{aligned}
e_{i,K} &\le \frac{{e}_{i,0}}{NT} +  \mathcal{O}\left( \frac{\bar R \sqrt{\log(NT) \log\left(\frac{2 SA T}{\delta}\right)}}{(1 - \gamma)^{2.5} \sqrt{NT}} \right).\\
\end{aligned}
\end{equation}
\end{theorem}

The proof of Theorem~\ref{thm:main_result} is deferred to the next section. We now highlight a few takeaways from the above result.

$\bullet$ \textbf{Near-Optimal Rates with Collaborative Speedup.} To interpret the guarantee in~\eqref{eqn:thm_main}, fix any agent \(i \in [N]\). 
For a single learner, prior work~\cite{Qu,Waiwright} has shown that \(Q\)-learning converges at the rate \(\tilde{\mathcal{O}}\!\left(\tfrac{1}{(1-\gamma)^{2.5}\sqrt{T}}\right)\) when trained on \(T\) samples. With \(N\) agents jointly interacting with the same environment \(\mc{M}\), the \emph{best possible rate} is \(\tilde{\mathcal{O}}\!\left(\tfrac{1}{\sqrt{NT}}\right)\), up to polynomial factors in $(1-\gamma).$ This matches the rate in Theorem~\ref{thm:main_result}, revealing that \texttt{VRDQ} enjoys \emph{near-optimal guarantees} that benefit from collaboration. \\
$\bullet$ \textbf{Poly-Log  Communication Overhead.} The total communication overhead (per agent) of \texttt{VRDQ} is $KL$, where $K$ is the number of epochs, and $L$ is the number of diffusion steps per epoch. From~\eqref{eqn:parameter_choice}, we note that $K$ and $L$ are logarithmic in both the number of samples $T$, and the number of agents $N$. This marks a significant reduction in communication cost relative to prior work on distributed/federated RL, which suffer from 
\(\mc{O}(T)\) dependence on the number of samples and/or \(\mc{O}(N)\) dependence on the number of agents~\cite{woo2023blessing,khodadadian,wang2023TMLR}. In summary, \texttt{VRDQ} \textit{achieves near-optimal statistical efficiency while drastically reducing communication costs}, a crucial property for large-scale multi-agent RL systems.\\
$\bullet$~\textbf{Effect of Network Topology.} From~\eqref{eqn:thm_main}, we note that the network topology does not directly influence the final convergence rate. However, it does influence the initial burn-in time needed for our guarantees to kick in. To see why, observe from~\eqref{eqn:parameter_choice} that the network structure, as captured by the parameter $\rho$, sets a lower bound on the diffusion period $L$. Moreover, we require $T \geq KL$, since we need $L \leq H = T/K$, where $H$ is the epoch length. Thus, the network-dependent quantity $\rho$ imposes a lower bound on the total run-time duration $T.$


\section{Simulation Results for \texttt{VRDQ}}
In this section, we evaluate the performance of \texttt{VRDQ} on a synthetic grid-world environment with $10$ states, $5$ actions, and discount factor $\gamma = 0.9$. All rewards belong to the interval $[0,1]$. The step size is set to $\alpha = 0.1$, time-steps to $T = 10^5$, and the failure probability to $\delta = 0.01$. 
\begin{figure}[t]
\begin{center}
\begin{tabular}{cc}
   \hspace{-5 mm}\includegraphics[width = 45 mm, height=30mm]{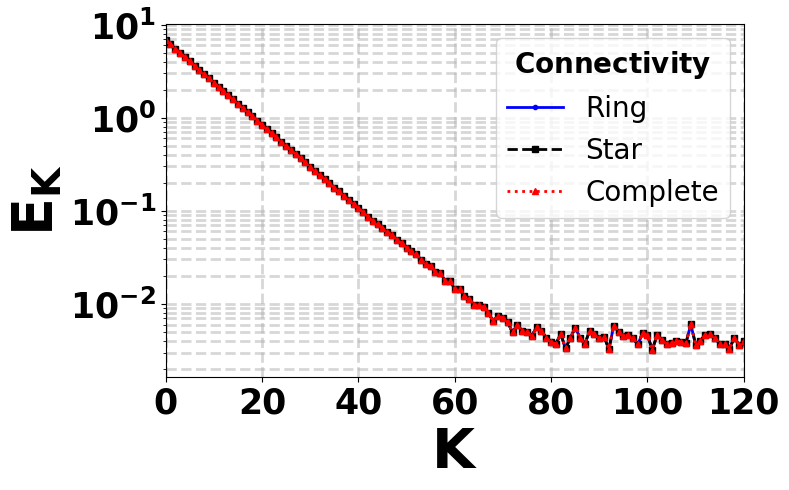}&\hspace{-4 mm}\includegraphics[width = 45 mm, height=30mm]{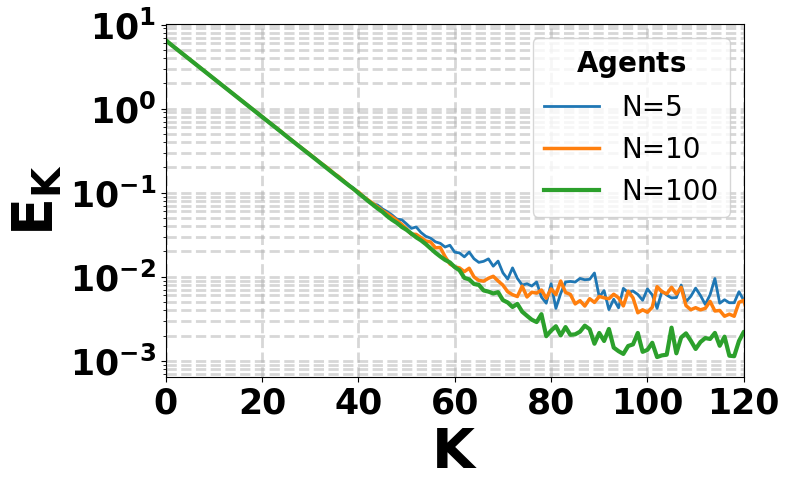}
    \end{tabular}
\end{center}
\caption{Plots of the $\ell_\infty$ error $E_K = \sum_{i\in [N]}\lVert Q_{i,K} - Q^* \rVert_{\infty}/N$ for \texttt{VRDQ} as a function of the number of epochs $K$, with varying number of agents \textbf{(Left)}, and under different network topologies with $N=100$ \textbf{(Right)}.}
\label{fig:sim2}
\vspace{-3mm}
\end{figure}
The left panel of Fig.~\ref{fig:sim2} illustrates the benefits of collaboration from \texttt{VRDQ}, showing a lower error floor with increasing $N$. In the right panel, we note that when the epoch length is chosen to be large enough, the network topology does not affect the convergence rate of \texttt{VRDQ}, complying with theory.   

\section{Q-Learning over Time-Varying Networks}
\label{sec:TimeVar}
In this section, we will briefly explain how our developments can be easily extended to account for a fairly general class of undirected time-varying networks. To do so, consider an undirected graph sequence $\{\mathcal{G}(0), \mathcal{G}(1), \ldots\}$, where at a given time-step $t$, $\mathcal{G}(t)=(\mathcal{V}, \mathcal{E}(t))$, with $\mathcal{E}(t)$ representing the (time-varying) edge set at time $t$. We say that $(i,j) \in \mathcal{E}(t)$ if and only if agents $i$ and $j$ can exchange information with each other at time $t$. Accordingly, the neighbor set of agent $i$ at time $t$ is defined as $\mathcal{N}_i(t) := \{j| (i,j) \in \mathcal{E}(t)\}.$ Let $W(t)$ be the mixing weight matrix at time $t$ that is consistent with $\mathcal{G}(t)$ in the sense that if $i \neq j$, and $(j,i) \notin \mathcal{E}(t)$, then $(W)_{ij}(t)=0$, where $(W)_{ij}(t)$ is the $ij$-th component of the matrix $W(t).$ To proceed, for any $b = 1, 2, \ldots,$ and $t \geq b-1$, let us define $W_b(t): = W(t) W(t-1) \cdots W(t-b+1).$ Following~\cite{nedic2017achieving}, we then impose the following standard assumptions on the sequence $\{W(t)\}$ of mixing matrices. 

\begin{assumption} 
\label{ass:Time_Var}
The sequence $\{W(t)\}$ satisfies the following: 
\begin{enumerate} 
\item[(i)] $W(t) \mathbf{1} = \mathbf{1},$ and $\mathbf{1}^{\top} W(t) = \mathbf{1}^{\top}, \forall t \geq 0.$
\item[(ii)] There exists a positive integer $B$ such that
$$\omega : = \sup_{t \geq B-1} \omega(t) < 1, \hspace{1mm} \textrm{where} \hspace{1mm} \omega(t):= \Vert W_B(t) - \frac{\mathbf{1} \mathbf{1}^{\top}}{N} \Vert_2, 
$$
\end{enumerate}
and $\omega(t)$ is defined as above for all $t \geq B-1.$
\end{assumption}

We note that the above assumptions have appeared extensively in the study of consensus algorithms; see~\cite{nedic2017achieving} and the references therein. To account for time-varying networks, the only minor modification to \texttt{VRDQ} is in the diffusion step~\eqref{eqn:diffusion_step}, which now takes the following form:
\begin{equation}\label{eqn:diffusion_step_TV}
    d_{i,k}^{(\ell+1)}(s,a) = \sum_{j \in \mathcal{N}_i(kH+\ell)} (W)_{ij}(kH+\ell) d_{j,k}^{(\ell)}(s,a), 
\end{equation}
where $\ell = 0, 1, \ldots, \bar{L}B-1,$ and $\bar{L}$ is a parameter to be specified shortly. Two points are worth noting. First, the $\ell$-th diffusion step within the $k$-th epoch corresponds to time-step $kH+\ell$; hence, $W(kH+\ell)$ governs the diffusion process at this step. Second, the agents run diffusion for only the first $\bar{L}B$ time steps within each epoch. We then have the following result for time-varying graphs. 

\begin{theorem} 
\label{thm:timevar}
Suppose Assumption~\ref{ass:Time_Var} holds, and \texttt{VRDQ} is run with the same choice of step-size $\alpha$ and epoch count $K$ as specified in~\eqref{eqn:parameter_choice}. Moreover, suppose that 
$$ \bar{L} = \ceil{{\log\!\Big(C\,N\sqrt{T}\sqrt{1-\gamma}\Big)}/{\log(1/\omega)}}, $$ 
where $\omega$ is as defined in Assumption~\ref{ass:Time_Var}, and $C$ is a suitable universal constant. Then, \texttt{VRDQ} provides the exact same guarantee as in~\eqref{eqn:thm_main}. 
\end{theorem}

The above result tells us that \texttt{VRDQ} continues to enjoy the same collaborative benefits as before, even for a fairly general sequence of time-varying networks. Moreover, the communication cost remains poly-logarithmic in $T$. To see why, we note that the communication overhead per agent is now $K\bar{L} B$, where $K$ and $\bar{L}$ remain logarithmic in both $N$ and $T$, and $B$ has no dependence on $T$. That said, we note that $B$ might very well depend on $N$. 
\section{Analysis}
In this section, we develop finite–time convergence guarantees for our
proposed algorithm \texttt{VRDQ}. As a first step, we pass from the agent-wise recursion in~\eqref{eqn:agent_update} to the network average \(\bar{Q}_k \coloneqq \tfrac{1}{N}\sum_{i=1}^{N} Q_{i,k}\), as follows:
\begin{equation}\label{eqn:average_Q}
\begin{aligned}
    \bar{Q}_{k+1} 
    &= (1 - \alpha) \cdot\underbrace{\frac{1}{N}\sum_{i=1}^{N} Q_{i,k}}_{:=\bar{Q}_k} + \alpha \cdot \underbrace{\frac{1}{N} \sum_{i=1}^{N} d_{i,k}^{(L)}}_{:=\bar{d}_k}.
    \end{aligned}
\end{equation} 
Next, we perform a simple decomposition of~\eqref{eqn:average_Q} to obtain the following recursive relation:
\begin{equation}\label{eqn:decomp}\medmath{
\begin{aligned}
    &\bar{Q}_{k+1} - Q^* = (1 - \alpha)(\bar{Q}_k - Q^*) + \alpha (\bar{d}_k - Q^*) \\
    &= (1 - \alpha)(\bar{Q}_k - Q^*) + \alpha (\mathcal{T}^* \bar{Q}_k - \mathcal{T}^* Q^*) + \alpha (\bar{d}_k - \mathcal{T}^* \bar{Q}_k),
\end{aligned}}
\end{equation}
where we used the fixed-point property of the Bellman optimality operator, namely \(\mathcal{T}^* Q^* = Q^*\)~\cite{suttonRL}. To proceed, we will use the following contractive property of the Bellman optimality operator~\cite{suttonRL}:
\begin{equation}\label{eqn:Bellmancontraction}
\lVert \mathcal{T}^* f_1 - \mathcal{T}^* f_2 \rVert_{\infty}\le \gamma \lVert f_1 - f_2\rVert_{\infty}, \forall f_1, f_2 \in \mathbb{R}^{S \times A}. 
\end{equation}
Taking the
\(\ell_\infty\)-norm on both sides of~\eqref{eqn:decomp}, and using the contraction property in~\eqref{eqn:Bellmancontraction}, we obtain the following bound on \(\bar{e}_{k+1} := \lVert \bar{Q}_{k+1} - Q^* \rVert_{\infty}\) as follows:
\begin{equation}\label{eqn:decomp_norm}
    \medmath{
    \begin{aligned}
       \bar{e}_{k+1} &\le (1-\alpha(1-\gamma))\bar{e}_{k} + \alpha \lVert  \bar{d}_k - \mathcal{T}^* \bar{Q}_k \rVert_{\infty}\\
        & \le (1-\alpha(1-\gamma))\bar{e}_{k} + \alpha \underbrace{\Big\lVert \bar{d}_k - \frac{1}{N}\sum_{i=1}^{N}\mathcal{T}_{i,k-1} Q_{i,k}\Big\rVert_{\infty}}_{(*)}  \\
        & + \alpha \underbrace{\Big\lVert \frac{1}{N}\sum_{i=1}^{N}\mathcal{T}_{i,k-1} Q_{i,k} - \mathcal{T}^{*}\bar{Q}_k \Big\rVert_{\infty}}_{(**)}.
    \end{aligned}}
\end{equation}
In analyzing the convergence of \texttt{VRDQ}, the key hurdle is to control the deviation 
\(\lVert \bar{d}_k - \mathcal{T}^* \bar{Q}_k \rVert_{\infty}\). This term reflects two sources of error: the \emph{diffusion error} \((*)\), analyzed in Lemma~\ref{lemma:deterministic_bound} and arising from the diffusion scheme in Algorithm~\ref{algo:Algo 2}, and the \textit{ operator estimation error} \((**)\), addressed in Lemma~\ref{lemma:operator_disc}. For controlling both sources of error, we need the following lemma that provides a bound on the iterates. 
\begin{lemma}(\textbf{Boundedness of Iterates})\label{lemma:bound_iterates}
    The following is true for the iterates generated by Algorithm \ref{algo:Algo 2}, $\forall k \geq 0$: 
\begin{equation}
 \lvert Q_{i,k}(s,a) \rvert \le \frac{\bar{R}}{1-\gamma}, \forall (s,a) \in \mathcal{S} \times \mathcal{A}, \forall i\in[N].
 \label{eqn:Q_bound}
\end{equation}
\end{lemma}
\begin{proof}
    We prove this result by induction. Since \(\bar{R} > 0\) and \(Q_{i,0}(s,a) = 0\) for all \((s,a) \in \mc{S} \times \mc{A}\) in Algorithm~\ref{algo:Algo 2}, the bound in~\eqref{eqn:Q_bound} holds trivially at \(k = 0\) for all agents. Suppose that the bound in~\eqref{eqn:Q_bound} holds for \emph{all agents} \(i \in [N]\), and for \emph{all epochs} up to some epoch $k \geq 0$. We need to show that the same bound applies to \(Q_{i,k+1}(s,a), \forall i \in [N], \forall (s,a) \in \mc{S} \times \mc{A}\). To that end, fix an agent $i$, a  state–action pair \((s,a) \), and recall that \(Q_{i,k+1}(s,a)\) is generated as per~\eqref{eqn:agent_update}, where each agent \(i\) computes \(d_{i,k}^{(L)}(s,a)\) according to~\eqref{eqn:diffusion_step} in epoch $k$. We proceed to bound \(d_{i,k}^{(L)}(s,a)\) as follows: 
\begin{equation}\label{eq:diffusion_bound_correct}
\begin{aligned}
\big|d_{i,k}^{(L)}(s,a)\big|
&= \Big|\sum_{j=1}^N (W^L)_{ij}\, d_{j,k}^{(0)}(s,a)\Big| \\
&\overset{(a)}{\le} \max_{j\in[N]} \big|d_{j,k}^{(0)}(s,a)\big| \\
&\overset{(b)}{=} \max_{j\in[N]} \big|\mathcal T_{j,k-1} Q_{j,k}(s,a)\big| \\
&\overset{(c)}{\le} \bar R + \gamma \max_{j\in[N]}\|Q_{j,k}\|_\infty \le \frac{\bar R}{1-\gamma}.
\end{aligned}
\end{equation}
Here, \((W^L)_{ij}\) denotes the \((i,j)\)-th entry of \(W^L\). Now, $(a)$ holds since powers of doubly-stochastic matrices remain doubly stochastic; $(b)$ uses the definition of $d_{j,k}^{(0)}$ in~\eqref{eqn:initial_diffusion_step}, and $(c)$ follows from the definition of $\cT_{j,k-1}$ in~\eqref{eqn:agent_wise_bellman_update}, together with the induction hypothesis.
Next, from~\eqref{eqn:agent_update}, we obtain
\begin{equation}
\begin{aligned}
    \big|Q_{i,k+1}(s,a)\big|
    &\le (1 - \alpha)\,\big|Q_{i,k}(s,a)\big| + \alpha\,\big|d_{i,k}^{(L)}(s,a)\big| \\
    &\overset{(\bullet)}{\le} (1 - \alpha)\,\frac{\bar R}{1-\gamma} + \alpha\,\frac{\bar R}{1-\gamma} = \frac{\bar R}{1-\gamma},
\end{aligned}
\end{equation}
where \((\bullet)\) uses the induction hypothesis and the bound in \eqref{eq:diffusion_bound_correct}. The proof follows by noting that the above argument applies identically to all agents and state-action pairs. 
\end{proof}
\vspace{-1mm}
Our next result helps control the diffusion error. 
\begin{lemma}(\textbf{Bounding Diffusion Error})\label{lemma:deterministic_bound}
For all \(k \in [K]\), the following bounds hold deterministically:
\begin{itemize}
    \item[\((a)\)]
    \(
    \Big\lVert \bar{d}_k - \tfrac{1}{N}\sum_{i=1}^{N}\mathcal{T}_{i,k-1} Q_{i,k}\Big\rVert_{\infty} 
    \;\leq\; \frac{\texttt{C}_1N\bar R\,\rho^L}{1-\gamma} := \Delta_1,
    \)
    \vspace{1 mm}
    \item[\((b)\)]
    \(
    \Bigl\lVert d_{i,k}^{(L)} - \bar{d}_k\Bigr\rVert_{\infty}
    \;\leq\; \frac{2\texttt{C}_1 N\bar R \rho^{L}}{1-\gamma},
    \)
\end{itemize}
where \(\bar{d}_k\) is as defined in~\eqref{eqn:average_Q}.
\end{lemma}

\begin{proof}
To establish item $(a)$, start by fixing an agent $i \in [N]$, an epoch $k \in [K]$, and a state-action pair $(s,a) \in \mathcal{S} \times \mathcal{A}$. Using~\eqref{eqn:initial_diffusion_step} and~\eqref{eqn:compact_diffusion_step}, we then have
\begin{equation}
\label{eqn:diff_err}
\medmath{
\begin{aligned}
& \vert d^{(L)}_{i,k}(s,a) - \tfrac{1}{N}\sum_{i=1}^{N}\mathcal{T}_{i,k-1} Q_{i,k}(s,a) \vert  = \left([W^L]_i - \frac{\mathbf{1}^{\top}}{N} \right) d^{(0)}_k(s,a) \\
&\overset{(a)}\leq \norm{[W^L]_i-\frac{1}{N}\mathbf{1}^\top} \Vert d^{(0)}_k(s,a) \Vert_1 \\
&\overset{(b)}\leq \frac{\texttt{C}_1N\bar R\,\rho^L}{1-\gamma} := \Delta_1,
\end{aligned}}
\end{equation}
where for (a), we used Holder's inequality, and for (b), we used~\eqref{eqn:fact_1} and $\max_{j\in[N]} \big|\mathcal T_{j,k-1} Q_{j,k}(s,a)\big| \leq \bar{R}/(1-\gamma)$ (from the analysis in Lemma~\ref{lemma:bound_iterates}). Since the analysis in~\eqref{eqn:diff_err} applies to all state-action pairs, we conclude that
\begin{equation}
\label{eqn:bound_2}
\Big\lVert d^{(L)}_{i,k} - \tfrac{1}{N}\sum_{i=1}^{N}\mathcal{T}_{i,k-1} Q_{i,k}\Big\rVert_{\infty} 
    \;\leq\; \frac{\texttt{C}_1N\bar R\,\rho^L}{1-\gamma} := \Delta_1.
\end{equation}
Noting that the above inequality applies identically to all agents establishes item $(a)$ in the statement of the lemma. For item $(b)$, we simply apply the triangle inequality, and combine item $(a)$ with~\eqref{eqn:bound_2}. 
\end{proof}

The following result controls the estimation error. 
\begin{lemma}(\textbf{Bounding Estimation Error})\label{lemma:operator_disc}
    With probability at least \(1-\delta\), the following bound holds for all \(k\in[K]\):
    \begin{equation}\label{eqn:aggregation_error}
        \underbrace{\Big\lVert \frac{1}{N}\sum_{i=1}^{N}\mathcal{T}_{i,k-1} Q_{i,k} - \mathcal{T}^{*}\bar{Q}_k \Big\rVert_{\infty}}_{(*)} \;\leq\; \Delta_2, \hspace{1mm} \textrm{where} \hspace{1mm} 
    \end{equation}
    \begin{equation}
        \Delta_2 \;:=\; {\mathcal{O}}\!\left(\frac{\bar{R} \sqrt{{\log(2 SAT/\delta)}}}{(1-\gamma)\sqrt{N H}}\right) \;+\; 
        \mathcal{O}\!\left(\frac{N\bar R \rho^{L}}{1-\gamma}\right).
        \label{eqn:operator_disc_bound}
    \end{equation}
\end{lemma}
\begin{proof}
To facilitate our analysis, we decompose the estimation error term defined in~\eqref{eqn:aggregation_error}, denoted by \((*)\), as $(*) \;\leq\; \mc{A}_1 + \mc{A}_2$, 
where
\begin{equation}\medmath{
    \begin{aligned}
        & \mc{A}_1 := \Big\lVert \frac{1}{N} \sum_{i=1}^N \Big[\mathcal{T}_{i,k-1} Q_{i,k} - \mathcal{T}^* Q_{i,k}\Big] \Big\rVert_{\infty},\\
        & \mc{A}_2 := \Big\lVert \frac{1}{N} \sum_{i=1}^N \Big[\mathcal{T}^* Q_{i,k} - \mathcal{T}^* \bar{Q}_{k} \Big] \Big\rVert_{\infty}.
    \end{aligned}}
\end{equation}
 The first term, \(\mc{A}_1\), captures the statistical error between the empirical Bellman operator in~\eqref{eqn:agent_wise_bellman_update} and the true Bellman optimality operator in~\eqref{eqn:bellman_optimality}. The second term, \(\mc{A}_2\), arises from the ``consensus gap" between $Q_{i,k}$ and $\bar{Q}_k$.

 \textbf{Bounding}\quad\({\mc{A}_1}\)\textbf{.} To bound $\mc{A}_1$, we will leverage concentration inequalities for sums of sub-Gaussian random variables.\footnote{A random variable $X \in \mathbb{R}$ is said to be sub-Gaussian with variance proxy $\sigma^2$ (or $\sigma$-sub-Gaussian) if its moment-generating function satisfies $\mathbb{E}[\exp(sX)] \leq \exp(s^2 \sigma^2/2), \forall s\in \mathbb{R}$~\cite{rigollet2023high}.} To see how this can be done, fix a particular state--action pair $(s,a) \in \mc{S} \times \mc{A}$. First, we decompose~\eqref{eqn:agent_wise_bellman_update} by invoking the empirical probability distribution in ~\eqref{eqn:empirical_prob_dist}: 
\begin{equation}\label{eqn:subgaussian_decomp}\medmath{
\begin{aligned}
    &\cT_{i,k-1} Q_{i,k}(s,a) = \cR(s,a) + \gamma \sum_{s' \in \cS} \cP_{i,k-1}(s' \mid s,a) \max_{a' \in \cA} Q_{i,k}(s',a')\\
    &=\frac{1}{H}\sum_{u=0}^{H-1}\underbrace{\left(\cR(s,a) + \gamma \sum_{s' \in \cS} \mathbf{1}_{i,k-1,u}^{s,a}(s') \max_{a' \in \cA} Q_{i,k}(s',a')\right)}_{X_u}.
\end{aligned}}
\end{equation}
Now observe that 
\begin{equation}\label{eqn:zeromean}\medmath{
\begin{aligned}
        &\mathbb{E}\!\left[ X_u \mid \{Q_{i, k}\}_{i\in[N]}\right]\\
        &{=} \mc{R}(s,a) + \gamma \sum_{s' \in \mc{S}} \mathbb{E}\!\Big[ \mathbf{1}_{i,k-1,u}^{s,a}(s') \mid \{Q_{i, k}\}_{i\in[N]}\Big]\max_{a'\in \mc{A}} Q_{i,k}(s',a') \\
        &\overset{(\bullet)}{=} \mc{R}(s,a) + \gamma \mathbb{E}_{s' \sim \mc{P}(\cdot|s,a)} \max_{a' \in \mc{A}} Q_{i,k}(s',a')\\
        &= \mathcal{T}^* Q_{i,k}(s,a).
        \end{aligned}}
    \end{equation}
As per~\eqref{eqn:agent_update}, the iterates $\{Q_{i,k}\}_{i \in [N]}$ depend only on the randomness realized up to the end of epoch $k-2$, and are therefore independent of the randomness in epoch $k-1$ under the i.i.d.\ synchronous sampling model. As a result, in $(\bullet)$, we used 
$\mathbb{E}\!\Big[ \mathbf{1}_{i,k-1,u}^{s,a}(s') \,\Big|\, \{Q_{i,k}\}_{i \in [N]} \Big] 
= \mathbb{E}\!\left[\mathbf{1}_{i,k-1,u}^{s,a}(s')\right] 
= \mc{P}(s' \mid s,a).$
Consequently, conditioned on $\{Q_{i,k}\}_{i \in [N]}$, each random variable $\{X_u - \mathcal{T}^* Q_{i,k}(s,a)\}$ is zero-mean for all $u \in [H]$. Moreover, under the i.i.d.\ synchronous sampling model, the samples $\{X_u\}_{u=0}^{H-1}$ generated within epoch $k-1$ are independent. Hence, the collection $\{X_u - \mathcal{T}^* Q_{i,k}(s,a)\}_{u=1}^H$ forms an i.i.d.\ sequence of zero-mean random variables, conditional on $\{Q_{i,k}\}_{i \in [N]}$. Furthermore, by invoking the boundedness of the Bellman optimality operator in~\eqref{eqn:bellman_optimality}, agent-wise Bellman update in~\eqref{eqn:agent_wise_bellman_update}, and local \(Q\)-updates from Lemma~\ref{lemma:deterministic_bound}, we obtain
\begin{equation}\label{eqn:almostsurebound}
\medmath{
\begin{aligned}
    \big\lvert X_u - \mathcal{T}^* Q_{i,k}(s,a) \big\rvert
    \leq \lvert X_u \rvert + \lvert \mathcal{T}^{*} Q_{i,k} \rvert \hspace{0.5mm} \leq \hspace{0.5mm} \frac{2 \bar{R}}{1 - \gamma} := \Gamma.
\end{aligned}}
\end{equation}
By the conditional zero-mean property established in~\eqref{eqn:zeromean}, together with the boundedness property in~\eqref{eqn:almostsurebound}, we note that the sequence $\{X_u - \mathcal{T}^* Q_{i,k}(s,a)\}$ forms an i.i.d.\ collection of $\Gamma$-sub-Gaussian random variables~\cite[Example~5.6]{lattimore2020bandit}. It also follows from~\eqref{eqn:subgaussian_decomp} that the sequence $\{\mathcal{T}_{i,k-1} Q_{i,k}(s,a) - \mathcal{T}^* Q_{i,k}(s,a)\}$ is itself sub-Gaussian with variance proxy $\Gamma^2/H$~\cite[Lemma~1.8]{rigollet2023high}. Conditional on $\{Q_{i,k}\}_{i \in [N]}$, the sole source of randomness in $\mathcal{T}_{i,k-1} Q_{i,k}(s,a)$ arises from the state transitions within the epoch, which are independent across agents under our sampling model, Thus, $\{\mathcal{T}_{i,k-1} Q_{i,k}(s,a) - \mathcal{T}^* Q_{i,k}(s,a)\}_{i=1}^N$ are i.i.d.\ $\Gamma/\sqrt{H}$-sub-Gaussian random variables~\cite[Lemma~1.8]{rigollet2023high}. Hence, averaging across agents further reduces the variance proxy by a factor of $N$. As a result, conditioned on $\{Q_{i,k}\}_{i \in [N]}$, the following event holds with probability at least $1 - \bar{\delta}$:
\begin{equation}\medmath{
\mc{E}:=\Bigg\{ 
\Bigg\lvert \frac{1}{N} \sum_{i=1}^N \Big[\mathcal{T}_{i,k-1} Q_{i,k}(s,a) - \mathcal{T}^* Q_{i,k}(s,a)\Big] \Bigg\rvert \le \frac{\texttt{C}_0\Gamma}{\sqrt{H}} \sqrt{\frac{\log(2/\bar{\delta})}{N}} \Bigg\}}, \nonumber
\end{equation}
where $\texttt{C}_0$ is some universal constant. Let $\mathbf{1}_{\mc{E}}$ be the indicator of the event $\mc{E}$. We then have 
\[
\mathbb{P}(\mc{E}) = \mathbb{E}[\mathbf{1}_{\mc{E}}] = \mathbb{E}\!\left[\mathbb{E}[\mathbf{1}_{\mc{E}} \mid \{Q_{i, k}\}_{i\in[N]}]\right] \;\geq\; 1 - \bar{\delta},
\]
where the last inequality follows from $\mathbb{P}(\mc{E} \mid \{Q_{i, k}\}_{i\in[N]}) \geq 1 - \bar{\delta}$.  
By taking a union bound over all state–action pairs \((s,a) \in \mc{S} \times \mc{A}\) and epochs \(k \in [K]\), with $K \leq T$, the following bound is guaranteed to hold simultaneously for all \((s,a)\) and \(k\) with probability at least \(1 - \delta\):
\begin{equation}\label{eqn:final_A_1_bound}\medmath{
    \Bigg\lvert \frac{1}{N} \sum_{i=1}^N \Big[\mathcal{T}_{i,k-1} Q_{i,k}(s,a) - \mathcal{T}^* Q_{i,k}(s,a)\Big] \Bigg\rvert \le \frac{\texttt{C}_0\Gamma}{\sqrt{H}} \sqrt{\frac{\log(2 SAT/\delta)}{N}}.
    }
\end{equation}
The definition of the $\infty$-norm implies the exact same bound as above on $\mc{A}_1.$

\textbf{Bounding}\quad\({\mc{A}_2}\)\textbf{.} To bound \(\mc{A}_2\), we start by controlling the error  \(\epsilon_{i,k} := Q_{i,k} - \bar{Q}_{k}\) by subtracting~\eqref{eqn:average_Q} from~\eqref{eqn:agent_update}: 
\begin{equation}\label{eqn:initial_error_for_A2}
Q_{i,k+1} - \bar{Q}_{k+1} = (1-\alpha)\big(Q_{i,k} - \bar{Q}_k\big) 
   + \alpha\big(d_{i,k}^{(L)} - \bar{d}_k\big). \\
\end{equation}
Unrolling~\eqref{eqn:initial_error_for_A2} over \(k\) epochs yields the following:
\begin{equation}
\epsilon_{i,k} = (1-\alpha)^k \epsilon_{i,0} 
   + \alpha \sum_{\tau=0}^{k-1} (1-\alpha)^{k-1-\tau} 
      \big(d_{i,\tau}^{(L)} - \bar{d}_{\tau}\big). 
\end{equation}

\noindent
As a consequence of the initialization of $Q_{i,0}$ in Algorithm~\ref{algo:Algo 2}, we have $\epsilon_{i,0} = 0$ for all $i \in [N]$. It follows that
\begin{equation}\label{eqn:epsilon}\medmath{
\begin{aligned}
\big\lVert \epsilon_{i,k} \big\rVert_\infty 
&= \big\lVert \alpha \sum_{\tau=0}^{k-1} (1-\alpha)^{k-1-\tau} 
  \left(  d_{i,\tau}^{(L)} - \bar{d}_{\tau} \right)  \big\rVert_\infty \\
&{\leq} \alpha \sum_{\tau=0}^{k-1} (1-\alpha)^{k-1-\tau} 
   \big\lVert d_{i,\tau}^{(L)} - \bar{d}_{\tau}  \big\rVert_\infty \\
&\overset{(\blacklozenge )}{\leq} \alpha \sum_{\ell=0}^{\infty} (1-\alpha)^{\ell} 
   \cdot \frac{\texttt{C}_1 N\bar{R}\rho^L}{1-\gamma} 
= \frac{\texttt{C}_1 N\bar{R}\rho^L}{1-\gamma},
\end{aligned}}
\end{equation}
where for \((\blacklozenge)\), we applied \((b)\) in Lemma~\ref{lemma:deterministic_bound}. The claim of Lemma~\ref{lemma:operator_disc} follows by combining the individual bounds on \(\mc{A}_1\) and \(\mc{A}_2\) that we derived above.
\end{proof}
We are now ready to prove our main result, Theorem~\ref{thm:main_result}.
\begin{proof}[Proof of Theorem \ref{thm:main_result}]   Lemmas~\ref{lemma:deterministic_bound} and \ref{lemma:operator_disc} inform us that there exists a \emph{good event} $\mc{\bar{\mc{E}}}$ of measure at least $1-\delta$, on which, \(\| \bar{d}_k - \mathcal{T}^* \bar{Q}_k \|_\infty \le \Delta := \Delta_1 + \Delta_2, \forall k =0, 1, \ldots, K-1, \) where $\Delta_1$ is as in item $(a)$ of Lemma~\ref{lemma:deterministic_bound}, and $\Delta_2$ is as in~\eqref{eqn:operator_disc_bound}. On this event, unrolling~\eqref{eqn:decomp_norm} yields:
\begin{equation}
\begin{aligned}
    \bar{e}_{K} &\le \underbrace{(1 - \alpha(1 - \gamma))^{K-1} \bar{e}_1}_{(*)} 
     + \underbrace{\sum_{k=1}^{K-1} \alpha (1 - \alpha(1 - \gamma))^{K-1 - k} \Delta}_{(**)}. 
\end{aligned}
\label{eqn:penult_bnd}
\end{equation}
The term \((**)\) can be further bounded as 
$$ (**)  \le \alpha {\Delta} \sum_{p=0}^{\infty} (1 - \alpha(1 - \gamma))^{p} = \frac{{\Delta}}{1-\gamma}.$$
Plugging the above bound in~\eqref{eqn:penult_bnd}, on the event $\mc{\bar{E}}$, we have 
\begin{equation}
\bar{e}_K \leq (1 - \alpha(1 - \gamma))^{K-1} \bar{e}_0 + \frac{\Delta}{(1-\gamma)}.
\label{eqn:final_bnd}
\end{equation}
In the above step, we used the fact that since $d_{i,0}^{(0)}(s,a)$ is initialized as zero, $Q_{i, 1}$'s will also be zero-vectors $\forall i \in [N]$, implying $\bar e_1=\bar e_0$. To arrive at the final bound in Theorem \ref{thm:main_result}, set \(\alpha\) according to~\eqref{eqn:parameter_choice}. Under this choice, we have
\begin{equation}\medmath{
    (*) = (1 - \alpha(1 - \gamma))^{K-1} \bar{e}_0
    \le \exp{(-\alpha(1-\gamma)(K-1))} \bar{e}_0 = \bar{e}_0/(NT),}
\nonumber
\end{equation}
where we used $(1-x) \leq \exp(-x),$ for $x \in (0,1)$. Using  the expressions for $\Delta_1$ and $\Delta_2$, we then obtain
\begin{equation}\label{eqn:penultimate_step}
    \bar{e}_K \leq \frac{\bar{e}_0}{NT} + \tilde{\mathcal{O}}\!\left(\frac{\bar{R}}{(1-\gamma)^{\frac{5}{2}}\sqrt{NH}}\right) \;+\; 
        \mathcal{O}\!\left(\frac{N\bar R \rho^{L}}{(1-\gamma)^2}\right).
\end{equation}
Substituting \(H = T/K\) and the choices of \(K\) and \(L\) from~\eqref{eqn:parameter_choice} into~\eqref{eqn:penultimate_step} yields the following bound, which holds with probability at least \(1-\delta\):

\begin{equation}
\bar{e}_K \leq \frac{\bar{e}_0}{N T}
+ \mathcal{O}\left(
\frac{\bar{R}\sqrt{\log(NT)\log\left(\tfrac{2SAT}{\delta}\right)}}
{(1-\gamma)^{5/2}\sqrt{N T}}
\right).\label{eqn:ek}
\end{equation}
\noindent Finally, we bound the agent-wise sub-optimality error as $e_{i,K} \leq \bar e_K + \norm{\epsilon_{i,K}}$. Using the bound on $\norm{\epsilon_{i,k}}$ in~\eqref{eqn:epsilon}, and the choice of $L$ in~\eqref{eqn:parameter_choice}, it is easy to verify that the bound on $\bar e_K$ in~\eqref{eqn:ek} also applies to $e_{i,k}$ (up to universal constants). This completes our proof of Theorem \ref{thm:main_result}.
\end{proof}
We now provide the proof for Theorem~\ref{thm:timevar}. 
\begin{proof} [Proof of Theorem \ref{thm:timevar}] Since the proof of Theorem~\ref{thm:timevar} shares the same structure as that of Theorem~\ref{thm:main_result}, we only elaborate on the key distinction that arises in controlling the diffusion error. To that end, fix an epoch $k$,  a state-action pair $(s,a)$, and notice from~\eqref{eqn:diffusion_step_TV} that
\begin{equation}
d^{(\bar{L}B)}_k(s,a)= W_B(k; \bar{L}) \cdots W_B(k;1) d^{(0)}_k(s,a),
\label{eqn:TV1}
\end{equation}
where we define $W_B(k;\ell) := W_B(kH+ \ell B-1), \ell =1, \ldots, \bar{L}.$ We claim that the following is true $\forall \ell \in [\bar{L}]$:
\begin{equation}
 z^{(\ell B)}_k(s,a) = \tilde{W}_B(k; \ell) \cdots \tilde{W}_B(k;1) d^{(0)}_k(s,a), 
\label{eqn:TV2}
\end{equation}
where $z^{(\ell B)}_k(s,a): = d^{(\ell B)}_k(s,a) - \frac{\mathbf{1} \mathbf{1}^{\top}}{N} d^{(0)}_k(s,a)$, and $\tilde{W}_B(k;\ell) := W_B(kH+ \ell B-1) - {\mathbf{1}\mathbf{1}^{\top}}/{N}$. Assuming this claim to be true for now, let us complete the rest of the analysis. Taking the 2-norm on both sides of~\eqref{eqn:TV2} with $\ell$ set to $\bar{L}$, and using item (ii) in Assumption~\ref{ass:Time_Var}, we obtain
\begin{equation}
    \Vert z^{(\bar{L} B)}_k(s,a) \Vert_{\infty} \leq \Vert z^{(\bar{L} B)}_k(s,a) \Vert_2 \leq \omega^{\bar{L}} \Vert d^{(0)}_k(s,a) \Vert_2. 
\label{eqn:TV3}
\end{equation}
Owing to the double-stochasticity of the sequence $\{W(k)\}$, the exact same bound on the iterates as derived in \eqref{eq:diffusion_bound_correct} continues to apply; thus, combined with~\eqref{eqn:TV3}, it is easy to then verify that
$ 
\Vert z^{(\bar{L} B)}_k(s,a) \Vert_{\infty} \leq G,
$
where $G= \mathcal{O}\left( \omega^{\bar{L}} \sqrt{N} \bar{R}/(1-\gamma) \right)$. Recalling that $d_{j,k}^{(0)}(s,a) = \mathcal{T}_{j,k-1}Q_{j,k}(s,a)$, and noting that the above argument applies identically to every state-action pair, we conclude that 
$$ \Vert d^{(\bar{L}B)}_{i,k} - \frac{1}{N} \sum_{j=1}^{N} \mathcal{T}_{j,k-1}Q_{j,k} \Vert_{\infty} \leq G.$$
From the above display, one can bound each of the objects in items $(a)$ and $(b)$ of Lemma~\ref{lemma:deterministic_bound} by $\mathcal{O}(G)$. The rest of the analysis is identical to that of Theorem~\ref{thm:main_result}. We now justify~\eqref{eqn:TV2} by inducting on $\ell$. The base case with $\ell=1$ follows directly from the relation $d^{(B)}_k(s,a) = W_B(k;1) d^{(0)}_k(s,a)$ by adding and subtracting ${\mathbf{1} \mathbf{1}^{\top}}/{N}$ to $W_B(k;1)$. Now suppose the claim in~\eqref{eqn:TV2} holds for all $\ell =1, 2, \ldots , p-1$, where $p = 2, \ldots, \bar{L}$. To extend the claim to $\ell = p$, starting from $  d^{(pB)}_k(s,a) = W_B(k;p) d^{((p-1)B)}_k(s,a)$, observe that  
\begin{equation}
    \begin{aligned}
z^{(pB)}_k(s,a) &= W_B(k;p) d^{((p-1)B)}_k(s,a) - \frac{\mathbf{1} \mathbf{1}^{\top}}{N} d^{(0)}_k(s,a)\\
&\overset{(a)}= W_B(k;p) z^{((p-1)B)}_k(s,a)\\
&\overset{(b)}= \tilde{W}_B(k; p) \cdots \tilde{W}_B(k;1) d^{(0)}_k(s,a) + \Psi,
    \end{aligned}
\end{equation}
where $\Psi = ({\mathbf{1} \mathbf{1}^{\top}}/{N} )  \tilde{W}_B(k; p-1) \cdots \tilde{W}_B(k;1) d^{(0)}_k(s,a)$. In the above steps, $(a)$ follows by noting that $ W_B(k;p)$ is also doubly-stochastic in light of double-stochasticity of $\{W(k)\}$, and $(b)$ uses the induction hypothesis. To complete the proof, we need to argue that $\Psi=0$. For this, noting that $p \geq 2$, observe: $\frac{\mathbf{1} \mathbf{1}^{\top}}{N} \tilde{W}_B(k; p-1)= \frac{\mathbf{1} \mathbf{1}^{\top}}{N} \left({W}_B(k; p-1) -  \frac{\mathbf{1} \mathbf{1}^{\top}}{N} \right) = \left(\frac{\mathbf{1} \mathbf{1}^{\top}}{N} - \frac{\mathbf{1} \mathbf{1}^{\top}}{N} \right) =0,$ where we used $\mathbf{1}^{\top} W_B(k;p-1) = \mathbf{1}^{\top}.$ This establishes $\Psi=0$, thereby completing the proof.  
\end{proof}

\section{Conclusion}
We introduced a novel approach for distributed Q-learning over static and time-varying networks, and showed that collaborative speedups in sample-complexity can be achieved with just a logarithmic communication overhead. In future work, we plan to extend our approach to account for function approximation and Markov sampling. 
\bibliographystyle{IEEEtran}
\bibliography{references.bib}

\end{document}